\documentclass[final,5p,times,twocolumn,authoryear]{elsarticle}
\usepackage{amsmath,amssymb,amsthm}
\usepackage{booktabs,multirow}
\usepackage{graphicx}
\usepackage{xcolor}
\usepackage[hidelinks]{hyperref}

\usepackage[section]{placeins}     

\newtheorem{proposition}{Proposition}
\newcommand{\vacf}{VACF}
\makeatletter
\def\ps@pprintTitle{\let\@oddhead\@empty\let\@evenhead\@empty\let\@oddfoot\@empty\let\@evenfoot\@empty}
\makeatother

\begin{document}
\onecolumn
\thispagestyle{empty}
\begin{center}
{\LARGE\bfseries Where Should the Gaussians Go? Reusable Label-Aware Allocation for Budget-Constrained Gaussian Splatting of Cinematic CT\par}
\vspace{10pt}
{\large Bowen Xiang, Dingjie Su, Jon S. Heiselman, Michael I. Miga\par}
\end{center}
\vspace{18pt}
\noindent{\Large\bfseries Graphical abstract}\par
\vspace{10pt}
\noindent\includegraphics[width=\textwidth]{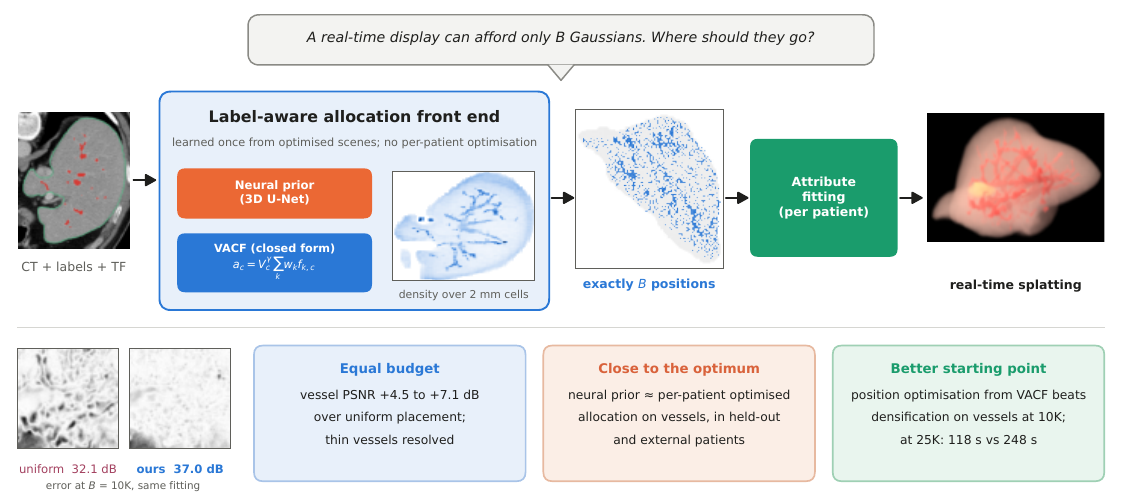}\par
\vspace{22pt}
\noindent{\Large\bfseries Highlights}\par
\vspace{6pt}
\begin{itemize}\setlength{\itemsep}{3pt}
\item Where Gaussians go under a fixed budget is learned once instead of per patient
\item A label-aware allocation front end with neural and closed-form implementations
\item On vessels, the neural prior matches per-patient optimised allocation, externally too
\item As initialisation, allocation beats per-patient densification on vessels at 10K
\item Rendering presets show when label-based and appearance-based allocation fails
\end{itemize}
\clearpage
\twocolumn
\setcounter{page}{1}
\begin{frontmatter}
\title{Where Should the Gaussians Go? Reusable Label-Aware Allocation for Budget-Constrained Gaussian Splatting of Cinematic CT}

\author[1]{Bowen Xiang\corref{cor1}}
\ead{bowen.xiang@vanderbilt.edu}
\author[2]{Dingjie Su}
\author[1]{Jon S. Heiselman}
\author[1]{Michael I. Miga}
\cortext[cor1]{Corresponding author.}
\address[1]{Department of Biomedical Engineering, Vanderbilt University, Nashville, TN, USA}
\address[2]{Department of Computer Science, Vanderbilt University, Nashville, TN, USA}

\begin{abstract}
Anatomical labels help 3D Gaussian splatting preserve vessels in cinematic renderings of CT, but the representation is still found by optimising each patient's scene, where densification decides where Gaussians go, while real-time displays cap their number. We address one part of this process: deciding where a fixed budget of $B$ Gaussians should be placed. We propose a reusable, label-aware allocation front end that maps a patient's CT, labels and transfer function to a density over 2\,mm cells and draws exactly $B$ positions from it without per-patient optimisation; the other Gaussian attributes are still fitted per patient. The front end is learned once from per-patient optimised scenes and implemented as a population-trained neural prior and an interpretable seven-parameter visibility-aware closed form (\vacf). With identical attribute fitting on 20 held-out patients at 10K--100K Gaussians, both raise vessel-masked PSNR over uniform allocation by 4.5--7.1\,dB and over pruning a dense voxel-aligned model by 1.6--3.3\,dB. The neural prior does not differ significantly from per-patient optimised allocation at any budget, and \vacf\ is within 0.6\,dB of it on vessels; on 19 external patients the neural prior again matches it on vessels and \vacf\ is within 0.7\,dB. As an initialisation for position optimisation, \vacf\ allocation beats per-patient densification on vessels at 10K Gaussians and, at 25K, reaches a higher vessel PSNR after 118\,s of optimisation than densification after 248\,s. Under altered rendering presets, \vacf\ transfers to new colours and lighting but not to a label-free transfer function, whereas the neural prior shows the opposite pattern.
\end{abstract}

\begin{keyword}
Gaussian splatting \sep cinematic rendering \sep anatomical labels \sep volume visualisation \sep liver vessels \sep mixed reality
\end{keyword}
\end{frontmatter}

\section{Introduction}\label{sec:intro}

Cinematic rendering simulates the transport of light through a CT volume classified by a transfer function (TF). It conveys depth, shape and the spatial relations of vessels and lesions better than conventional direct volume rendering and is increasingly used for surgical planning, patient communication and education \citep{eid2017cinematic,dappa2016cinematic}, but path tracing takes far longer than interactive frame rates allow on mobile and head-mounted hardware. 3D Gaussian splatting (3DGS) \citep{kerbl2023gaussian} removes this barrier at display time: a set of anisotropic Gaussians fitted to path-traced images rasterises at hundreds of frames per second. Label-aware Gaussian splatting (LAGS) \citep{lags} showed that anatomical labels are valuable in this process: using them during optimisation preserves fine structures such as vessels, and the resulting scenes can be explored on a standalone headset.

The Gaussian representation itself, however, is still found by optimising each patient's scene. Starting from an initial set of Gaussians, densification and pruning decide, over tens of thousands of iterations on path-traced views, where Gaussians are needed and how many there are, and this search is repeated for every patient and every TF \citep{lags,clipgs}. The resulting count is a by-product of thresholds, whereas the devices that motivate the approach impose a fixed budget $B$ on the number of Gaussians that can be stored, streamed, sorted and blended per frame. Feed-forward models such as Render-FM \citep{renderfm} avoid per-patient optimisation but place one Gaussian per foreground voxel, so their count follows the size of the organ (50K--800K Gaussians in Render-FM). Neither approach answers where a fixed number of Gaussians should go.

We address this one part of the process. Our premise is that where a scene needs its Gaussians is largely predictable from the anatomy and the rendering setup, so that it can be learned once from optimised scenes of other patients and reused. We propose a \emph{label-aware allocation front end}: from a patient's CT, labels and TF it computes a density over 2\,mm cells and draws exactly $B$ Gaussian positions from it, in seconds and without per-patient optimisation (Fig.~\ref{fig:problem}). The front end takes over the placement decisions that densification makes; it does not estimate colour, opacity or shape. Every complete scene reported in this paper still has these attributes fitted to path-traced views of the patient (about 2\,min on a GPU), optionally together with the positions; we do not generate complete 3DGS scenes in one forward pass.

\begin{figure*}[!tp]\centering
\includegraphics[width=\textwidth]{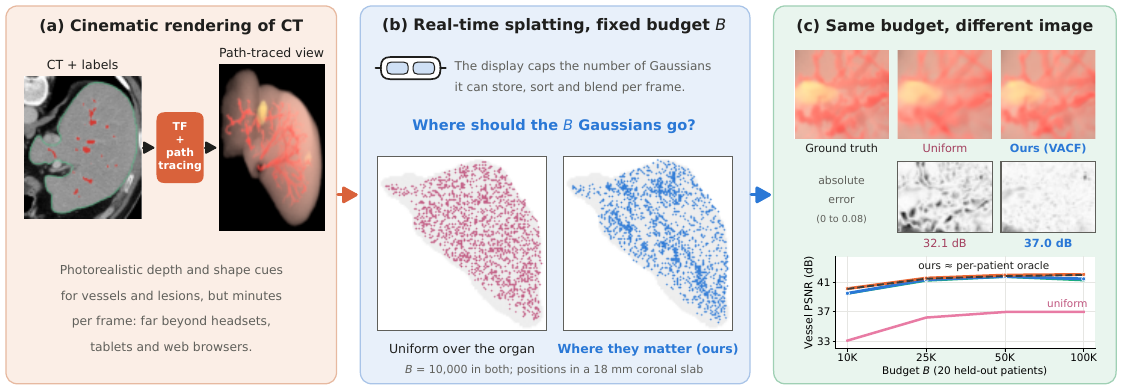}
\caption{The problem. (a) Cinematic rendering path-traces a CT volume classified by a transfer function (TF); it conveys depth and shape but takes minutes per frame. (b) 3D Gaussian splatting renders such scenes in real time, but head-mounted displays, tablets and browsers cap the number of Gaussians $B$, so their positions must be chosen. Shown are $B=10{,}000$ positions (in an 18\,mm coronal slab) drawn uniformly over the organ and by our visibility-aware closed form (\vacf), which concentrates them on vessels and the visible surface. (c) With identical fitting of all other attributes, the two allocations give different images: crops of a held-out view, absolute error and crop PSNR (MSD test patient 372); and vessel-masked PSNR against the budget on 20 held-out patients, where the neural prior matches the per-patient oracle, \vacf\ comes close, and uniform allocation stays 4.5--7.1\,dB lower.}\label{fig:problem}
\end{figure*}

We implement the front end in two ways that share their supervision and objective and differ in model class. Both are trained to reproduce \emph{oracle densities}, the per-cell counts of Gaussians in per-patient optimised scenes of training patients. A \emph{population-trained neural prior}, a 3D U-Net on ten cell-level inputs, can represent any dependence of these densities on the inputs. A \emph{visibility-aware closed form} (\vacf) has seven parameters: six weights for the label fractions and TF gradient of a cell, and an exponent on the visibility of the cell from the viewing sphere. It is inspectable and its failure modes can be stated in advance. We evaluate both under a protocol in which all methods share the same attribute fitting and differ only in where the $B$ Gaussians are, and let the results establish their respective strengths.

Our contributions are:
\begin{enumerate}
\item \textbf{A reusable, label-aware allocation front end} for budget-constrained Gaussian splatting of cinematic CT, which separates the placement of Gaussians from the fitting of their attributes, together with an equal-budget evaluation protocol that isolates the effect of placement.
\item \textbf{Two implementations learned once from per-patient optimised scenes}: a population-trained neural prior and an interpretable visibility-aware closed form.
\item \textbf{Evidence on held-out and external patients.} At equal budgets, both raise anatomical rendering quality far above uniform placement and above pruning a dense voxel-aligned model; the neural prior does not differ significantly from the per-patient optimised allocation, on 20 held-out patients and on 19 external patients from another centre. As an initialisation, the closed-form allocation lets position optimisation surpass per-patient densification on vessels at 10K Gaussians, and at 25K reach a higher vessel PSNR in less than half the optimisation time. Rendering presets show the limits of each implementation: the closed form fails when appearance no longer follows the labels, and the neural prior fails when the TF changes the contrast it was trained on.
\item \textbf{Supporting analysis.} Per-patient optimised allocations at 25K--200K Gaussians differ about as much as two optimisation seeds do, as a simple power-law model predicts, which supports drawing different budgets from one density per patient; and the main systematic difference between optimised allocations and anatomy-only rules is the visibility of each cell.
\end{enumerate}

\paragraph{Scope of comparisons.} All baselines are our own implementations under the common protocol. The dense voxel-aligned baseline adopts Render-FM's placement (one Gaussian per foreground cell) but is fitted per patient; it is not the Render-FM network. We do not evaluate the official LAGS pipeline: our per-patient models use generic absolute-gradient densification \citep{absgs}, and label-aware training as in LAGS is complementary to, and could start from, our allocation.

\section{Related work}\label{sec:related}

\paragraph{Gaussian splatting for medical volumes.} LAGS \citep{lags} introduced label-aware densification and rasterisation for path-traced medical volumes with a mixed-reality viewer; its budget study (50K--200K Gaussians; Table~IV of \citealp{lags}) reported gains of 2.7--9.2\,dB over vanilla 3DGS at equal counts, attributed to a better distribution of Gaussians, which motivates the present work. ClipGS and XClipGS add interactive clipping \citep{clipgs,xclipgs}; DDGS-CT targets X-ray rendering from CT \citep{ddgsct}. All of these optimise each scene. Render-FM \citep{renderfm} is a feed-forward model trained on 991 TotalSegmentator CTs that predicts one 6D Gaussian per foreground voxel from the CT, masks and TF, without per-patient optimisation but without control of the count. Expressive kernels (6DGS, 7DGS, Universal Beta Splatting) reduce counts as a side effect of the kernel \citep{6dgs,7dgs,ubs}; allocation is complementary to the kernel.

\paragraph{Compact Gaussians and budget control.} For natural scenes, compactness is pursued by pruning with importance scores such as LightGaussian's global significance \citep{lightgaussian}, by growth under a budget as in Taming 3DGS \citep{taming}, by Markov-chain relocation \citep{mcmc}, by densification on absolute gradients \citep{absgs} and by codebook compression \citep{compact3dgs}. All operate per scene by optimisation. We use significance pruning of a dense voxel-aligned model as a baseline, and budget-capped growth to build per-patient references.

\paragraph{Feed-forward Gaussian prediction.} Generalisable models predict Gaussians from images or volumes in one pass \citep{pixelsplat,renderfm}; they typically tie Gaussians to pixels or voxels. Our front end decides positions under a budget and could precede such a decoder; in this work, attributes are fitted per patient.

\paragraph{Optimal point density and visibility.} The allocation model of Section~\ref{sec:model} is a finite-cell form of asymptotic quantisation theory, in which the optimal density of quantisation points is a power of the source density \citep{zador,gersho}. Importance-driven and visibility-driven volume rendering \citep{viola2005importance,correa2011visibility} use related notions of how visible a structure is, but to classify or emphasise structures, not to represent a volume with a fixed number of primitives.

\section{Method}\label{sec:method}

Fig.~\ref{fig:workflow} summarises the method: oracle allocations from per-patient optimisation on training patients (Section~\ref{sec:teachers}), two allocators learned from them (Sections~\ref{sec:learned}--\ref{sec:relation}), and their deployment on a new patient (Section~\ref{sec:protocol}).

\begin{figure*}[!tp]\centering
\includegraphics[width=\textwidth]{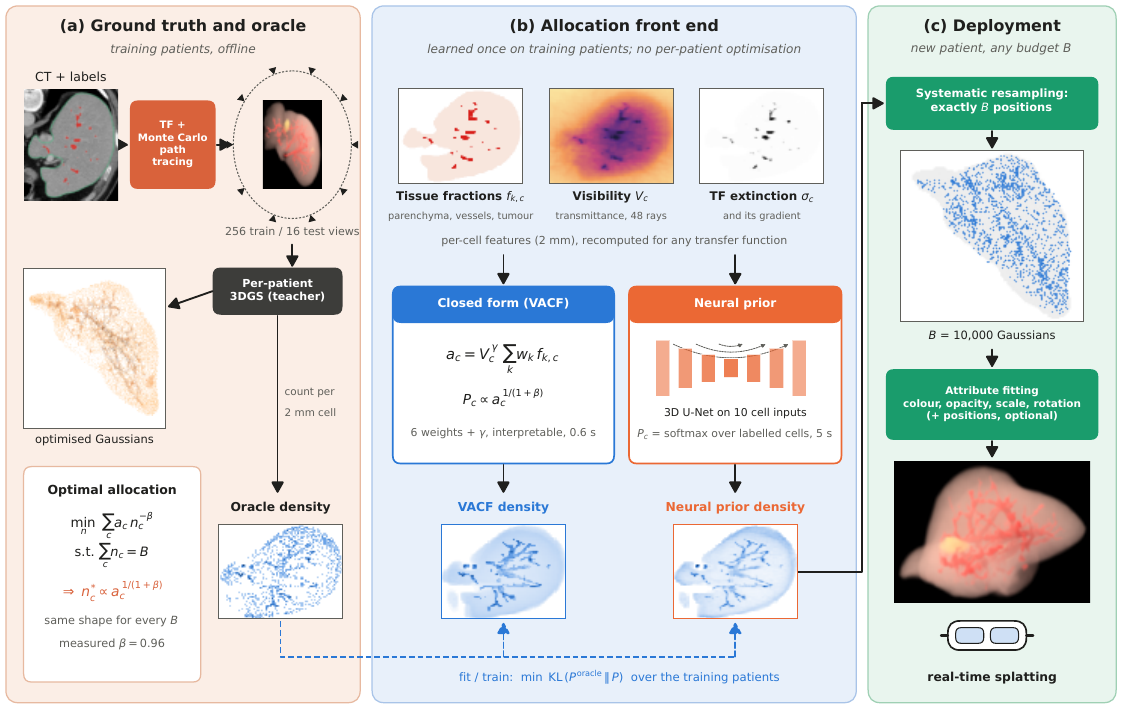}
\caption{Method. (a) For each training patient, CT and labels are classified by a TF and path-traced from 256 training views on a sphere; a per-patient 3DGS model (teacher) optimised on these views defines the oracle allocation, the count of its Gaussians per 2\,mm cell. A simple distortion model (Proposition~\ref{prop:alloc}) suggests that one density per patient can serve a range of budgets. (b) The allocation front end predicts this density from per-cell features without per-patient optimisation, implemented as a visibility-aware closed form (\vacf) with six tissue weights and a visibility exponent, or as a neural prior (3D U-Net). Both are fitted to the oracle densities of the training patients by minimising the KL divergence (dashed) and then frozen. (c) For a new patient and any budget $B$, exactly $B$ positions are drawn from the density; colour, opacity, scale and rotation, and optionally the positions, are then fitted to path-traced views of that patient, and the result is rendered by splatting in real time. All images are data of MSD test patient 372: CT slice with labels (vessels red, tumour yellow, liver outline green), teacher Gaussians (coronal projection, darker = more opaque), cell features (one axial 2\,mm slice), densities (22\,mm axial slab, square-root scale), sampled positions (18\,mm coronal slab) and a rendering after \vacf\ allocation with optimised positions at 10K Gaussians.}\label{fig:workflow}
\end{figure*}

\subsection{Setting}\label{sec:setting}
The region of interest (the liver with a 6\,mm margin) is partitioned into cells $c\in\mathcal C$, 2\,mm cubes unless stated otherwise. An \emph{allocator} maps a patient's CT, labels and TF to a density $P_c$ over cells; exactly $B$ Gaussian positions are then drawn from $P$. What is computed once and what per patient is as follows. Learned once on training patients: the allocator's parameters. Per patient, without optimisation: cell features (seconds) and the allocation (under a second for \vacf, 5\,s for the neural prior). Per patient, by optimisation: the remaining Gaussian attributes, fitted to path-traced views of that patient, as for any 3DGS scene.

\subsection{Oracle allocations from per-patient optimisation}\label{sec:teachers}
For each training patient a \emph{teacher} 3DGS model is optimised on 256 path-traced training views (Section~\ref{sec:gt}) with gsplat \citep{gsplat}: 20K initial Gaussians uniformly in the foreground, 30K iterations, spherical harmonics of degree 3, loss $0.8\,L_1+0.2\,(1-\mathrm{SSIM})$, and densification by the absolute screen-space gradient \citep{absgs} with threshold $4\times10^{-4}$. The \emph{oracle} density of a patient is the normalised count of its teacher's Gaussians (opacity $>0.05$) per cell: where optimisation chose to put Gaussians. Teachers are defined for every patient in the study, so that the oracle also serves as the per-patient reference on test patients. We additionally train \emph{budget-capped} teachers that reach exactly $B$ Gaussians (at each refinement step the top-$k$ Gaussians by accumulated gradient are densified, with $k$ following a linear schedule from the initial count to $B$); they serve as the reference ``per-case 3DGS at $B$'', which optimises positions together with all other attributes, and for the analysis of Section~\ref{sec:res_theory}.

\subsection{An allocation model}\label{sec:model}
A simple model motivates two design choices, one density per patient for a range of budgets and a visibility factor. Suppose the rendering distortion after attribute fitting is additive over cells with a local power law,
\begin{equation}\label{eq:distortion}
D(\mathbf n)\approx\sum_{c} a_c\,n_c^{-\beta},\qquad \beta>0,
\end{equation}
where $n_c$ is the number of Gaussians in cell $c$ and $a_c\ge0$ the complexity of the cell, which depends on the anatomy, the TF, the lighting and the visibility of the cell.

\begin{proposition}\label{prop:alloc}
In the continuous relaxation, the minimiser of \eqref{eq:distortion} subject to $\sum_c n_c=B$ is $n_c^{*}=B\,a_c^{1/(1+\beta)}\big/\sum_{c'}a_{c'}^{1/(1+\beta)}$. In particular, the normalised allocation $n_c^*/B$ does not depend on $B$, and uniform allocation is optimal only if all $a_c$ are equal.
\end{proposition}
\begin{proof}
The Lagrangian $\sum_c a_c n_c^{-\beta}+\lambda(\sum_c n_c-B)$ is stationary when $\beta a_c n_c^{-\beta-1}=\lambda$, i.e.\ $n_c\propto a_c^{1/(1+\beta)}$; each term is strictly convex in $n_c>0$, so this is the global minimum. Optimality of uniform allocation requires equal $a_c$ by the same condition.
\end{proof}

\emph{Visibility.} A Gaussian that misrepresents the radiance of a cell produces a pixel error that reaches the camera attenuated by the transmittance $T$ between the cell and the camera; for a squared-error metric its contribution scales with $T^2$. Averaged over a viewing sphere, the complexity of a cell should therefore grow with its visibility, and we model it as an intrinsic term $\tilde a_c$ times a power of the mean transmittance,
\begin{equation}\label{eq:vis}
a_c\approx V_c^{\gamma}\,\tilde a_c,\qquad V_c=\frac{1}{|\Omega|}\sum_{u\in\Omega}\exp\!\Big(-\!\int_0^{\infty}\!\sigma(x_c+tu)\,dt\Big),
\end{equation}
where $\sigma$ is the TF extinction and $\Omega$ a set of 48 directions. Because $0\le T\le1$, the mean of $T^2$ lies between $V_c^2$ and $V_c$, so the squared-error argument alone would suggest $1\le\gamma\le2$. We treat $\gamma$ as an empirical parameter and fit it.

The model ignores occlusion between neighbouring Gaussians and assumes a common exponent for all cells. We therefore use it only as motivation and test what matters empirically: whether optimised allocations are budget-invariant (Section~\ref{sec:res_theory}) and whether the exponent and the visibility term affect image quality (Section~\ref{sec:res_ablation}).

\subsection{Neural allocation prior}\label{sec:learned}
The neural prior is a 3D U-Net (MONAI \citep{monai}; channels 24--192, three downsamplings, residual units) whose output logits, normalised by a softmax over the cells that carry a tissue label, are the allocation density. Its ten cell-level inputs are CT intensity; the fractions of parenchyma, tumour and vessels in three calibre classes ($r<1.5$, $1.5$--$3$, $\ge3$\,mm; Section~\ref{sec:metrics}); the TF extinction and its gradient magnitude; the visibility $V_c$ of \eqref{eq:vis}, computed by marching 48 rays per cell through the TF extinction on the GPU; and depth below the liver surface. It is trained on the 56 training patients to minimise
\begin{equation}\label{eq:kl}
\sum_{\text{patients}}{\rm KL}\big(\hat P^{\rm oracle}\,\big\|\,P\big),
\end{equation}
the negative log-likelihood of the teachers' Gaussian positions up to a constant, where $\hat P^{\rm oracle}$ is the raw, unsmoothed oracle density (Appendix~\ref{app:notes}). Training uses random $48^3$-cell crops for 300 epochs (AdamW, learning rate $3\times10^{-4}$, cosine schedule), with model selection by the whole-volume likelihood on 10 validation patients. Input ablations retrain the network without groups of channels; the variant ``without TF inputs'' drops the extinction and its gradient but keeps visibility, which is computed from the extinction.

\subsection{Visibility-aware closed form (\vacf)}\label{sec:vacf}
\vacf\ restricts the allocator to the form suggested by Section~\ref{sec:model}, with the label fractions $f_{k,c}$ (parenchyma, three vessel calibres, tumour) and the TF-gradient magnitude $g_c$ as intrinsic complexity:
\begin{equation}\label{eq:vacf}
P_c\propto\Big(V_c^{\gamma}\big(\textstyle\sum_{k}w_kf_{k,c}+w_gg_c\big)\Big)^{1/(1+\beta)},
\end{equation}
with weights on the simplex, $\gamma\ge0$, and $\beta$ fixed from a power-law fit (Section~\ref{sec:res_theory}). The six weights and $\gamma$ are fitted once on the training patients with the objective \eqref{eq:kl}, by Nelder--Mead in a softmax parameterisation of the simplex. Because the visibility is recomputed from the TF extinction, \vacf\ adapts to a new TF without refitting; its tissue terms, however, assume that what is rendered follows the labels.

\subsection{How the two implementations relate}\label{sec:relation}
Both allocators are trained on the same targets with the same objective and are applied in the same way; they differ in hypothesis class. \vacf\ is a pointwise function of a subset of the neural prior's inputs (label fractions, TF gradient, visibility), a fixed monotone function of a weighted sum modulated by visibility, so the neural prior can in principle represent it. Its six weights and visibility exponent can be read and compared across cell sizes, and the conditions under which it will fail follow from its form. The neural prior can represent interactions and spatial context that the closed form cannot, such as the arrangement of vessels around a cell, at the cost of 2.7 million parameters and of learning whatever its inputs encode about the training TF. Which of these trade-offs matters is an empirical question, answered in Sections~\ref{sec:res_main}--\ref{sec:res_presets}.

\subsection{Deployment and evaluation protocol}\label{sec:protocol}
Exactly $B$ positions are drawn by systematic resampling of the density, which yields integer counts per cell summing to $B$, with stratified jitter inside each cell. Gaussians are initialised at these positions with an isotropic scale from the distance to their three nearest neighbours and colour and opacity from the TF; colour, opacity, scale and rotation are then optimised for 15K iterations on the patient's training views without densification or pruning. For evaluation, positions stay frozen, and the protocol is identical for every allocator, so that differences in image quality are due to placement alone. In deployment, positions may be optimised as well (Section~\ref{sec:res_positions}).

\subsection{Baselines and references}\label{sec:baselines}
\emph{Uniform voxels}: $B$ foreground voxels drawn uniformly. \emph{TF density}: density proportional to the TF extinction. \emph{Anatomy heuristic}: label weights parenchyma/vessel/tumour $=1/4/2$ plus the normalised TF gradient, the strongest hand-designed rule in our pilot study. \emph{Segmentation-free}: Frangi vesselness \citep{frangi1998} and multiscale CT gradient inside the rendered region, fitted like \vacf. \emph{Dense model, top-$B$}: a dense voxel-aligned per-patient model with one Gaussian per foreground 2\,mm cell, as in Render-FM's placement, fitted with frozen positions, from which the top-$B$ Gaussians are kept by opacity or by LightGaussian global significance \citep{lightgaussian} and then refitted under the common protocol. The \emph{oracle density} and \emph{per-case 3DGS at $B$} require per-patient optimisation and are references, not competitors.

\section{Experimental setup}\label{sec:setup}

\subsection{Cinematic ground truth}\label{sec:gt}
CT volumes are resampled to 1\,mm isotropic voxels and cropped to the liver with a 6\,mm margin. A TF maps CT and labels to extinction $\sigma$ and RGB single-scattering albedo (Section~\ref{sec:presets}). Images are rendered with Mitsuba~3 \citep{mitsuba3} (volumetric path tracing in a heterogeneous medium, Henyey--Greenstein phase function with $g=0.2$, environment map plus one directional key light, $512\times512$ pixels, 35$^\circ$ field of view). Cameras lie on a sphere that frames the region: 256 training views at 1024 samples per pixel (spp) and 16 held-out test views at 2048\,spp.

\subsection{Data}\label{sec:data}
\emph{Development and internal test.} We use the hepatic-vessel task of the Medical Segmentation Decathlon (MSD) \citep{msd}: 303 portal-venous contrast-enhanced CT scans with annotations of hepatic vessels and tumours. Because resampling thick slices to 1\,mm only interpolates, we keep scans with slice thickness $\le2.5$\,mm (86 of 303; 194 scans have 5\,mm slices and none lie between 2.5 and 3.75\,mm). Liver masks are obtained with TotalSegmentator \citep{totalseg}. Patients are split into 56 training, 10 validation and 20 test patients.

\emph{External test.} 3D-IRCADb-01 \citep{ircadb} contains 20 contrast-enhanced CTs from another centre with expert segmentations. We exclude one scan with 4\,mm slices (slice thickness of the remaining 19: 1.0--2.5\,mm), map the portal and hepatic venous systems to the MSD vessel label, restricted to the liver dilated by 5\,mm because MSD does not annotate the extrahepatic vena cava, and map all liver tumours to the tumour label; the liver mask is again taken from TotalSegmentator. The annotation protocol differs from MSD (for example, five patients have no tumour and the venous system is not delineated in all patients). IRCADb was used once, for the final evaluation, with all allocators frozen.

\subsection{Rendering presets}\label{sec:presets}
\emph{Preset A}, used for all fitting, training and model selection, is a label-aware TF: parenchyma $\sigma=0.012\,{\rm mm^{-1}}$, vessels $0.45$ and tumour $0.10$ with reddish vessels and tan parenchyma, parenchyma and vessel extinction modulated by CT intensity (40--200\,HU), and class occupancies smoothed by a 0.6-voxel Gaussian to avoid staircases. \emph{Preset B} keeps a label-aware TF but changes everything else: denser bluish parenchyma ($\sigma=0.022$), less dense whitish vessels ($0.30$), green tumour ($0.16$), a shifted intensity window (60--220\,HU), and a cooler key light from the opposite side. \emph{Preset C} is label-free: inside the liver, extinction and albedo are piecewise-linear functions of the CT intensity relative to the patient's median liver intensity (after Gaussian smoothing restricted to the liver), so that enhancing structures more than about 25\,HU above the parenchyma appear as dense red vessels and hypodense tissue appears faint yellow, whether or not they are labelled; lighting is as in preset A. A relative window is needed because contrast enhancement varies widely between patients (median vessel intensity 116--209\,HU on the test set). Presets B and C were rendered for the first 10 test patients and never used for any decision. For each preset the oracle uses teachers trained on that preset's images, and the allocators' inputs (extinction, gradient, visibility) are recomputed under that preset.

\subsection{Model selection and use of the test data}\label{sec:selection}
All design decisions (TF, rendering protocol, teacher settings, the form of \vacf, its fitted parameters and the network) were made on training and validation patients. The MSD test set was used in two rounds. A first set of frozen allocators (an anatomy-only closed form fitted to smoothed teacher targets, among others) was evaluated first; its failure and the oracle gap it revealed motivated the visibility analysis, which was carried out on training patients only, and the final \vacf\ and network were developed on training and validation patients and frozen (with checksums) before the second evaluation. After that evaluation we found that the neural prior's inference normalised its density over every cell with non-zero extinction instead of the labelled cells used by its training loss, which placed about 20\% of its budget in cells without any tissue label, where teachers place about 1\%. We corrected the inference to match training, without retraining or any other change, and report the corrected prior throughout; the uncorrected prior was 0.54\,dB below the oracle on vessels at 10K. All later experiments on the MSD test set (component and input ablations, cell size, positions, perceptual metrics) added methods or metrics but did not change the frozen allocators. Two kinds of test data were never seen by any decision: the external IRCADb patients and presets B and C.

\subsection{Metrics and statistics}\label{sec:metrics}
For each test view we compute PSNR and SSIM and, from projected label maps obtained by ray-marching the label occupancies with the same extinction, region masks: foreground (accumulated opacity $>0.05$), vessels (projected vessel weight $>0.3$), vessel-calibre strata (visible vessel pixels dominated by one radius class), parenchyma and tumour. The radius of a vessel voxel is the largest distance to the vessel wall within its $7\times7\times7$-voxel ($\pm3$\,mm) neighbourhood, which reaches the centreline radius for vessels up to about 4\,mm in radius and exceeds 3\,mm for all larger ones, so every voxel of a vessel receives its correct class; non-vessel voxels take the class of the nearest vessel voxel. A thin branch within 3\,mm of a larger vessel therefore inherits the larger class, so the thin stratum ($r<1.5$\,mm) contains only thin branches away from larger vessels: it can miss thin segments near bifurcations but does not include parts of larger vessels. Region PSNRs pool squared error over the 16 test views of a patient. Because the clinical purpose of the rendering is to show the anatomy, we report vessel-masked and thin-vessel ($r<1.5$\,mm) PSNR alongside foreground PSNR. We also report LPIPS \citep{lpips} (AlexNet features, full images) and a vessel-structure metric, centreline ridge recall (CR80): reference centrelines are the skeleton of the ground-truth ridge map (scale-normalised Hessian ridge strength at 1--2.5 pixels, maximum over colour channels) inside the projected vessel mask; CR80 is the fraction of centreline pixels at which the rendering's ridge strength reaches 80\% of the ground truth's, overall and for thin vessels. It measures whether a branch is drawn with its contrast, which PSNR does not separate from small errors elsewhere. Comparisons are paired over patients (never over views): two-sided Wilcoxon signed-rank tests with Holm correction \citep{holm1979} within each comparison family (one pair of methods across metrics and budgets), and 95\% percentile-bootstrap confidence intervals of the mean difference (10,000 resamples). Budgets are $B\in\{10\text{K},25\text{K},50\text{K},100\text{K}\}$ on the MSD test set and $\{10\text{K},25\text{K},50\text{K}\}$ elsewhere. With 10 patients (presets), the smallest attainable Holm-adjusted $p$ is about 0.02. Not finding a significant difference does not establish equivalence; we therefore report confidence intervals alongside every comparison. All paired comparisons at every budget are listed in the Supplementary Material.

\subsection{Implementation}\label{sec:impl}
Rendering used Mitsuba 3.9 (\texttt{cuda\_ad\_rgb}); 3DGS used gsplat 1.6 with PyTorch 2.14. Experiments ran on an RTX 5070\,Ti laptop GPU and RTX 4090 GPUs. Median times on an RTX 4090: cell features including visibility about 7\,s per patient, allocation with \vacf\ 0.6\,s and with the neural prior 4.9\,s, and attribute fitting with frozen positions 2.0\,min; a teacher takes about 8\,min. All of these come in addition to path tracing the training views, which every method in this paper requires. Repeating a fit with the same seed changed vessel-masked PSNR by 0.03\,dB on average (maximum 0.36\,dB over 320 repeated fits), which bounds the run-to-run noise of the protocol.

\section{Results}\label{sec:results}

\subsection{Anatomical rendering quality at equal budgets}\label{sec:res_main}
Tables~\ref{tab:main}--\ref{tab:paired} and Fig.~\ref{fig:budget}a--c summarise the 20 held-out MSD patients.

\begin{table*}[!tp]
\caption{Held-out MSD test set ($n=20$ patients), preset A. PSNR (dB) in projected vessel pixels, thin-vessel pixels ($r<1.5$\,mm) and the foreground, for budgets of 10K--100K Gaussians. All allocators share the same attribute-fitting protocol (positions frozen). Bold: best value excluding the references. $^\dagger$References: allocation or scene from per-patient 3DGS optimisation.}\label{tab:main}
\centering\footnotesize\setlength{\tabcolsep}{3.6pt}
\begin{tabular}{l cccc cccc cccc}
\toprule
& \multicolumn{4}{c}{Vessels} & \multicolumn{4}{c}{Thin vessels ($r<1.5$\,mm)} & \multicolumn{4}{c}{Foreground} \\
\cmidrule(lr){2-5}\cmidrule(lr){6-9}\cmidrule(lr){10-13}
Allocator & 10K & 25K & 50K & 100K & 10K & 25K & 50K & 100K & 10K & 25K & 50K & 100K \\
\midrule
Uniform voxels & 33.07 & 36.21 & 36.98 & 36.99 & 31.78 & 33.90 & 34.63 & 34.86 & 35.90 & 37.76 & 38.16 & 38.10 \\
TF density & 39.54 & 41.24 & 41.64 & 41.22 & 37.00 & 39.97 & 41.23 & 40.69 & 37.58 & 39.67 & 40.51 & 40.36 \\
Segmentation-free & 34.86 & 38.36 & 39.81 & 39.69 & 33.00 & 36.46 & 37.99 & 37.99 & 37.12 & 39.28 & 39.81 & 39.56 \\
Anatomy heuristic & 39.47 & 41.24 & 41.77 & 41.29 & 37.10 & 40.03 & 41.35 & 40.75 & 38.22 & 40.18 & 40.71 & 40.43 \\
Dense model, top-$B$ opacity & 24.32 & 32.86 & 36.58 & 35.66 & 24.14 & 31.35 & 34.56 & 33.83 & 24.94 & 31.72 & 37.26 & 37.21 \\
Dense model, top-$B$ significance & 36.86 & 39.59 & 40.21 & 38.98 & 33.65 & 36.49 & 37.56 & 36.61 & 38.46 & 39.93 & 40.05 & 39.25 \\
\midrule
\vacf\ (ours) & 39.53 & 41.39 & 41.83 & 41.50 & 39.09 & 41.51 & 42.00 & 41.25 & 39.41 & 40.68 & 40.92 & 40.61 \\
Neural prior (ours) & \textbf{40.16} & \textbf{41.60} & \textbf{42.01} & \textbf{42.10} & \textbf{39.64} & \textbf{41.84} & \textbf{42.38} & \textbf{42.42} & \textbf{40.13} & \textbf{40.93} & \textbf{41.06} & \textbf{41.02} \\
\midrule
Oracle density$^\dagger$ & 40.13 & 41.52 & 41.89 & 42.03 & 39.41 & 41.66 & 42.57 & 42.90 & 40.14 & 40.88 & 41.10 & 41.16 \\
Per-case 3DGS$^\dagger$ & 40.33 & 41.80 & 42.22 & 42.42 & 41.47 & 42.86 & 43.21 & 43.11 & 41.05 & 41.58 & 41.51 & 41.36 \\
\bottomrule
\end{tabular}
\end{table*}

\begin{figure*}[!tp]\centering
\includegraphics[width=\textwidth]{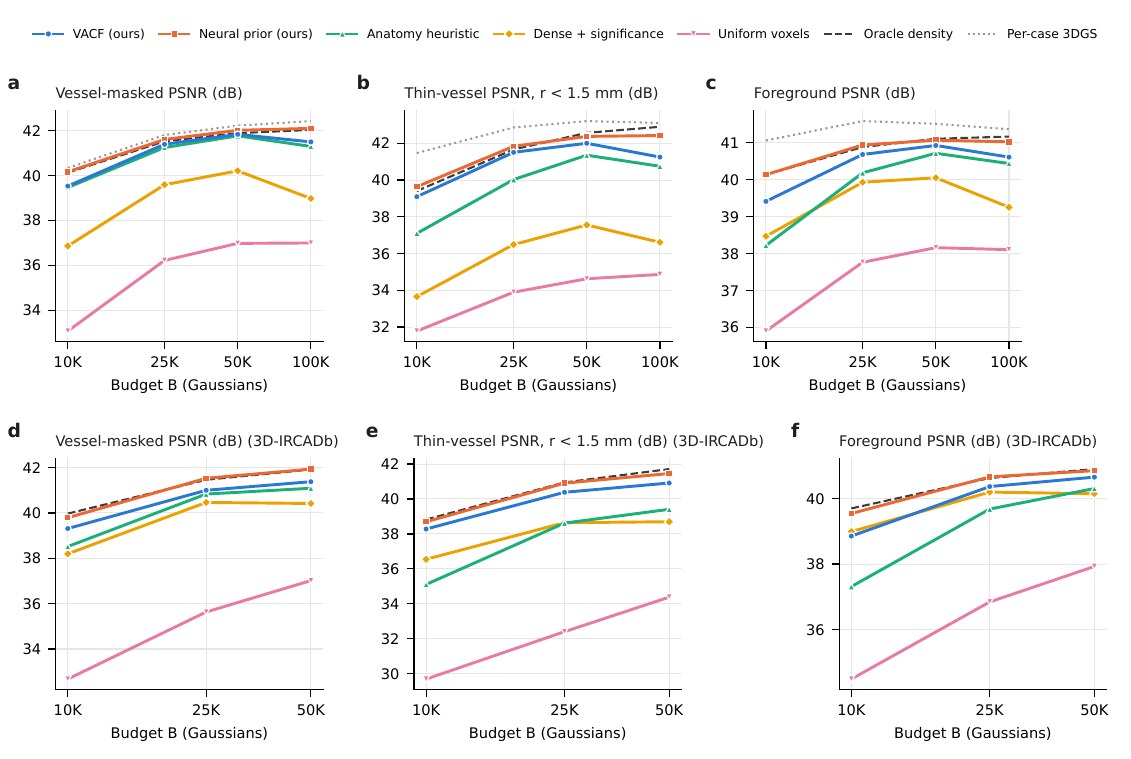}
\caption{Image quality as a function of the budget $B$ (log scale) after fitting all attributes at the allocated positions with the same protocol for every method. Top (a--c): 20 held-out MSD patients; bottom (d--f): 19 external 3D-IRCADb patients. (a, d) Vessel-masked PSNR, (b, e) thin-vessel PSNR ($r<1.5$\,mm), (c, f) foreground PSNR. The neural prior tracks the per-patient oracle density (dashed); \vacf\ is close below it; both stay well above uniform voxel allocation and pruning of a dense voxel-aligned model by LightGaussian significance, and their advantage over the anatomy heuristic is largest on thin vessels and at small budgets. Dotted (a--c): per-case 3DGS at budget $B$, which also optimises positions (Section~\ref{sec:res_positions}). Means over patients; paired statistics in Tables~\ref{tab:paired} and \ref{tab:external}.}\label{fig:budget}
\end{figure*}

\paragraph{Allocators versus baselines.} Both allocators are far above uniform voxels and pruning a dense model. The neural prior improves vessel-masked PSNR over uniform voxels by 7.1, 5.4, 5.0 and 5.1\,dB at 10K, 25K, 50K and 100K, and \vacf\ by $+6.46$ $[6.01,6.90]$, $+5.18$, $+4.86$ and $+4.50$\,dB. Over the dense model pruned by LightGaussian significance, \vacf\ gains $+2.67$ $[1.94,3.39]$, $+1.80$, $+1.63$ and $+2.52$\,dB on vessels and 4.5--5.4\,dB on thin vessels. Keeping the most opaque Gaussians of the dense model collapses (24.3\,dB on vessels at 10K), because opacity favours dense vessel interiors and leaves the rest of the liver empty; even significance, which uses per-patient optimisation, is a poor allocator, because a Gaussian that was important when all its neighbours existed is not the best choice once most of them are removed. The segmentation-free closed form recovers only part of the gap between uniform voxels and the label-based allocators (34.9\,dB on vessels at 10K).

\paragraph{The neural prior reaches the oracle.} The neural prior is not significantly different from the oracle density in any of twelve comparisons (three metrics, four budgets; mean differences between $-0.48$ and $+0.23$\,dB, all Holm-adjusted $p\ge0.7$). It is above the anatomy heuristic by $+0.70$ $[0.46,0.93]$\,dB on vessels, $+2.54$ $[2.19,2.86]$\,dB on thin vessels and $+1.92$ $[1.78,2.06]$\,dB on the foreground at 10K, and above \vacf\ at every budget (vessels $+0.64$, $+0.21$, $+0.18$ and $+0.60$\,dB; all $p<0.05$). Across three training seeds its mean vessel PSNR at 10K varies by 0.02\,dB (standard deviation).

\begin{table*}[!tp]
\caption{Paired differences on the MSD test set ($n=20$), in dB: mean over patients [95\% bootstrap CI]. $^{*}$Holm-adjusted Wilcoxon signed-rank $p<0.05$ within each comparison (3 metrics $\times$ 4 budgets; 25K and 100K in the Supplementary Material).}\label{tab:paired}
\centering\footnotesize\setlength{\tabcolsep}{3pt}
\begin{tabular}{l cc cc cc}
\toprule
& \multicolumn{2}{c}{Vessels} & \multicolumn{2}{c}{Thin vessels} & \multicolumn{2}{c}{Foreground} \\
\cmidrule(lr){2-3}\cmidrule(lr){4-5}\cmidrule(lr){6-7}
Comparison & 10K & 50K & 10K & 50K & 10K & 50K \\
\midrule
\vacf\ $-$ anatomy heuristic & $+0.06$ {\scriptsize[-0.20, +0.33]} & $+0.07$ {\scriptsize[-0.05, +0.16]} & $+2.00$$^{*}$ {\scriptsize[+1.59, +2.38]} & $+0.65$$^{*}$ {\scriptsize[+0.42, +0.88]} & $+1.19$$^{*}$ {\scriptsize[+1.08, +1.31]} & $+0.21$$^{*}$ {\scriptsize[+0.13, +0.30]} \\
\vacf\ $-$ dense, top-$B$ signif. & $+2.67$$^{*}$ {\scriptsize[+1.94, +3.39]} & $+1.63$$^{*}$ {\scriptsize[+1.12, +2.22]} & $+5.44$$^{*}$ {\scriptsize[+4.28, +6.62]} & $+4.45$$^{*}$ {\scriptsize[+3.38, +5.63]} & $+0.94$$^{*}$ {\scriptsize[+0.56, +1.34]} & $+0.87$$^{*}$ {\scriptsize[+0.63, +1.13]} \\
\vacf\ $-$ oracle density & $-0.60$$^{*}$ {\scriptsize[-0.71, -0.48]} & $-0.05$ {\scriptsize[-0.16, +0.05]} & $-0.31$$^{*}$ {\scriptsize[-0.45, -0.15]} & $-0.57$ {\scriptsize[-1.15, -0.13]} & $-0.73$$^{*}$ {\scriptsize[-0.80, -0.65]} & $-0.18$ {\scriptsize[-0.34, -0.06]} \\
Neural prior $-$ anatomy heuristic & $+0.70$$^{*}$ {\scriptsize[+0.46, +0.93]} & $+0.25$$^{*}$ {\scriptsize[+0.14, +0.37]} & $+2.54$$^{*}$ {\scriptsize[+2.19, +2.86]} & $+1.02$$^{*}$ {\scriptsize[+0.81, +1.26]} & $+1.92$$^{*}$ {\scriptsize[+1.78, +2.06]} & $+0.35$$^{*}$ {\scriptsize[+0.26, +0.45]} \\
Neural prior $-$ \vacf & $+0.64$$^{*}$ {\scriptsize[+0.47, +0.81]} & $+0.18$$^{*}$ {\scriptsize[+0.05, +0.34]} & $+0.54$$^{*}$ {\scriptsize[+0.32, +0.82]} & $+0.37$$^{*}$ {\scriptsize[+0.12, +0.70]} & $+0.72$$^{*}$ {\scriptsize[+0.63, +0.82]} & $+0.14$$^{*}$ {\scriptsize[+0.06, +0.24]} \\
Neural prior $-$ oracle density & $+0.04$ {\scriptsize[-0.09, +0.18]} & $+0.13$ {\scriptsize[+0.02, +0.24]} & $+0.23$ {\scriptsize[-0.03, +0.59]} & $-0.19$ {\scriptsize[-0.48, +0.03]} & $-0.00$ {\scriptsize[-0.07, +0.07]} & $-0.04$ {\scriptsize[-0.16, +0.06]} \\
Per-case 3DGS $-$ oracle density & $+0.21$$^{*}$ {\scriptsize[+0.07, +0.35]} & $+0.34$$^{*}$ {\scriptsize[+0.16, +0.50]} & $+2.06$$^{*}$ {\scriptsize[+1.79, +2.40]} & $+0.64$$^{*}$ {\scriptsize[+0.46, +0.91]} & $+0.92$$^{*}$ {\scriptsize[+0.79, +1.05]} & $+0.41$$^{*}$ {\scriptsize[+0.28, +0.53]} \\
\bottomrule
\end{tabular}
\end{table*}

\paragraph{\vacf.} \vacf\ improves thin-vessel PSNR over the anatomy heuristic by $+2.00$ $[1.59,2.38]$, $+1.48$, $+0.65$ and $+0.50$\,dB and foreground PSNR by $+1.19$, $+0.50$, $+0.21$ and $+0.17$\,dB (all Holm-adjusted $p<0.02$), while vessel-masked PSNR, which is dominated by large vessels, is unchanged ($+0.06$ to $+0.20$\,dB, n.s.). It is within $-0.60$, $-0.13$, $-0.05$ and $-0.54$\,dB of the oracle on vessels and within $-0.31$, $-0.16$, $-0.57$ and $-1.65$\,dB on thin vessels; its largest gaps are on the foreground at 10K ($-0.73$\,dB) and on thin vessels at 100K.

\paragraph{Perceptual and structural metrics.} LPIPS and centreline recall agree with the PSNR results (Table~\ref{tab:perceptual}): uniform voxels are worst, and the neural prior and \vacf\ are closest to the oracle (LPIPS at 10K: 0.021 for uniform voxels, 0.016 for the anatomy heuristic, 0.013 for \vacf, 0.011 for the neural prior and the oracle). At 10K, thin-vessel CR80 is 0.45 for uniform voxels, 0.72 for the pruned dense model, 0.88 for the anatomy heuristic, 0.95 for \vacf\ and 0.96 for both the neural prior and the oracle: with either allocator, 95\% of thin-vessel centreline pixels are drawn with at least 80\% of their true contrast, against 88\% for the heuristic ($+0.07$ $[0.06,0.09]$ for \vacf, $p<10^{-5}$). In individual crops (Fig.~\ref{fig:qual}), uniform allocation blurs small branches and leaves errors along every vessel edge, whereas differences among the anatomy-aware allocators at 10K are subtle (the pruned dense model is the best frozen-position allocator in one of the two crops shown), which is why we rely on patient-level statistics.

\begin{table}[!tbp]
\caption{Perceptual and vessel-structure metrics on the MSD test set ($n=20$). LPIPS (AlexNet; lower is better) over the full image, and centreline ridge recall CR80: the fraction of ground-truth vessel centreline pixels whose ridge contrast in the rendering reaches 80\% of that in the ground truth, for thin and all vessels.}\label{tab:perceptual}
\centering\footnotesize\setlength{\tabcolsep}{2.6pt}
\begin{tabular}{l cc cc cc}
\toprule
& \multicolumn{2}{c}{LPIPS $\downarrow$} & \multicolumn{2}{c}{CR80 thin $\uparrow$} & \multicolumn{2}{c}{CR80 all $\uparrow$} \\
\cmidrule(lr){2-3}\cmidrule(lr){4-5}\cmidrule(lr){6-7}
Allocator & 10K & 50K & 10K & 50K & 10K & 50K \\
\midrule
Uniform voxels & 0.021 & 0.013 & 0.445 & 0.771 & 0.550 & 0.841 \\
TF density & 0.017 & 0.011 & 0.877 & 0.978 & 0.936 & 0.982 \\
Anatomy heuristic & 0.016 & 0.010 & 0.877 & 0.981 & 0.933 & 0.984 \\
Dense, top-$B$ signif. & 0.015 & 0.009 & 0.720 & 0.883 & 0.804 & 0.934 \\
\midrule
\vacf\ (ours) & 0.013 & 0.009 & 0.952 & 0.985 & 0.952 & 0.985 \\
Neural prior (ours) & \textbf{0.011} & \textbf{0.008} & \textbf{0.963} & \textbf{0.988} & \textbf{0.968} & \textbf{0.988} \\
\midrule
Oracle density$^\dagger$ & 0.011 & 0.008 & 0.959 & 0.992 & 0.967 & 0.990 \\
\bottomrule
\end{tabular}
\end{table}

\begin{figure*}[!tp]\centering
\includegraphics[width=\textwidth]{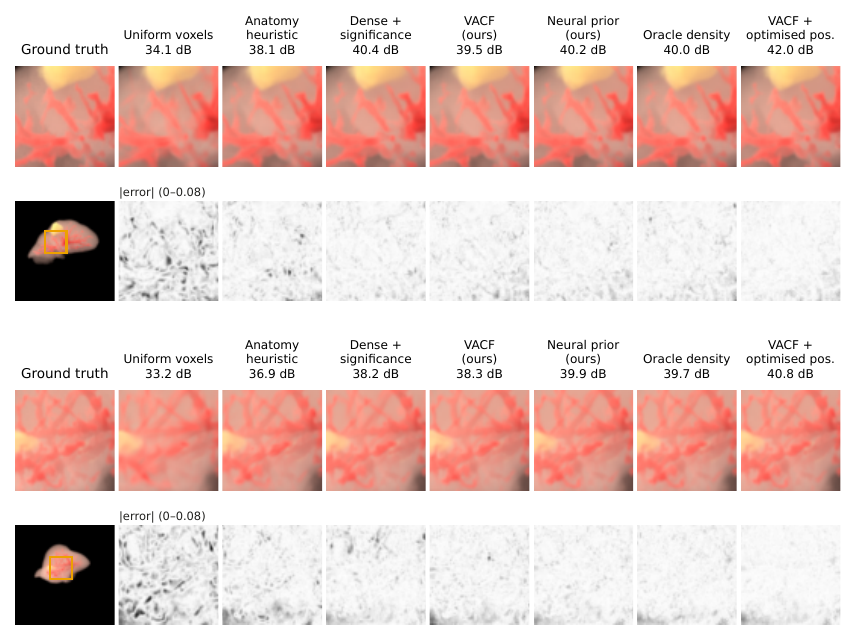}
\caption{Thin-vessel regions of two MSD test patients rendered with 10K Gaussians (preset A). First column: ground-truth crop (top) and the full view with the crop location (bottom). Other columns: renderings after allocation and attribute fitting (top; crop PSNR in dB) and the absolute error averaged over colour channels (bottom; white 0 to black 0.08). Uniform allocation blurs branches and leaves errors along every vessel edge; differences among the anatomy-aware allocators are subtle in single crops (patient-level statistics in Tables~\ref{tab:main}--\ref{tab:paired}). The last column shows \vacf\ followed by joint optimisation of positions. Crops are chosen automatically as the $112\times112$-pixel window with the most thin-vessel pixels in the view.}\label{fig:qual}
\end{figure*}

\paragraph{The remaining gap is placement within cells.} Per-case 3DGS at the same budget, which optimises positions jointly with all attributes, is above the oracle density by 0.21--0.39\,dB on vessels and by up to 2.06\,dB on thin vessels at 10K. Since the oracle density is sampled and refitted with the same frozen-position protocol as every allocator, this gap reflects sub-cell placement, which a density over 2\,mm cells cannot express; Section~\ref{sec:res_positions} shows that optimising positions after allocation closes it.

\subsection{External patients}\label{sec:res_external}
On the 19 IRCADb patients (Table~\ref{tab:external}, Fig.~\ref{fig:budget}d--f), acquired at another centre and annotated differently, the conclusions hold, with larger margins over the heuristic. The neural prior is within $-0.18$ to $+0.07$\,dB of the oracle on vessels and $-0.25$ to $-0.04$\,dB on thin vessels (n.s.), and $-0.15$\,dB on the foreground at 10K (Holm-adjusted $p=0.04$; n.s.\ at 25K and 50K); it exceeds the anatomy heuristic by $+1.27$ $[0.99,1.58]$\,dB on vessels and $+3.59$ $[2.91,4.30]$\,dB on thin vessels at 10K. \vacf\ improves over the heuristic by $+0.79$ $[0.50,1.11]$, $+0.17$ and $+0.29$\,dB on vessels, $+3.17$ $[2.52,3.84]$, $+1.76$ and $+1.51$\,dB on thin vessels and $+1.54$, $+0.69$ and $+0.35$\,dB on the foreground (all $p<0.01$), and is 0.46--0.66\,dB below the oracle on vessels. Uniform voxels are 7.3\,dB below the oracle on vessels at 10K. Pruning the dense model is relatively stronger on IRCADb than on MSD, but \vacf\ still exceeds it on thin vessels by 1.7--2.2\,dB (significant at 25K and 50K) and at 50K on all three metrics.

\begin{table*}[!tp]
\caption{External test on 3D-IRCADb-01 ($n=19$ patients; different centre, scanners and annotation protocol; never used for any decision). All allocators were fitted or trained on MSD training patients and applied unchanged. PSNR in dB; CR80 as in Table~\ref{tab:perceptual}. Bold: best value excluding the reference; $^\dagger$reference (per-patient 3DGS optimisation).}\label{tab:external}
\centering\footnotesize\setlength{\tabcolsep}{4pt}
\begin{tabular}{l ccc ccc ccc c}
\toprule
& \multicolumn{3}{c}{Vessels} & \multicolumn{3}{c}{Thin vessels ($r<1.5$\,mm)} & \multicolumn{3}{c}{Foreground} & CR80 thin \\
\cmidrule(lr){2-4}\cmidrule(lr){5-7}\cmidrule(lr){8-10}\cmidrule(lr){11-11}
Allocator & 10K & 25K & 50K & 10K & 25K & 50K & 10K & 25K & 50K & 10K \\
\midrule
Uniform voxels & 32.67 & 35.63 & 37.02 & 29.69 & 32.40 & 34.38 & 34.49 & 36.85 & 37.93 & 0.420 \\
TF density & 38.88 & 40.59 & 41.15 & 35.31 & 38.25 & 39.83 & 36.75 & 39.13 & 40.10 & 0.856 \\
Anatomy heuristic & 38.52 & 40.83 & 41.09 & 35.11 & 38.62 & 39.41 & 37.31 & 39.68 & 40.30 & 0.836 \\
Dense model, top-$B$ significance & 38.20 & 40.46 & 40.42 & 36.54 & 38.64 & 38.70 & 38.99 & 40.20 & 40.15 & 0.878 \\
\midrule
\vacf\ (ours) & 39.31 & 41.00 & 41.38 & 38.28 & 40.38 & 40.91 & 38.85 & 40.36 & 40.65 & 0.944 \\
Neural prior (ours) & \textbf{39.79} & \textbf{41.53} & \textbf{41.93} & \textbf{38.69} & \textbf{40.90} & \textbf{41.46} & \textbf{39.54} & \textbf{40.66} & \textbf{40.85} & \textbf{0.948} \\
\midrule
Oracle density$^\dagger$ & 39.97 & 41.46 & 41.92 & 38.82 & 40.95 & 41.71 & 39.70 & 40.63 & 40.90 & 0.958 \\
\bottomrule
\end{tabular}
\end{table*}

\subsection{Allocation as initialisation for position optimisation}\label{sec:res_positions}
In deployment the allocated positions need not stay fixed. We therefore asked whether allocation still matters when positions are optimised together with the other attributes, and how the result compares with per-case 3DGS, which grows its Gaussians from scratch by densification (Table~\ref{tab:positions}, Fig.~\ref{fig:positions}). This experiment was run for \vacf, the oracle and two baselines. Optimising positions (15K iterations, no densification) improves every allocation, \vacf\ by $+1.93$\,dB on vessels and $+3.18$\,dB on thin vessels at 10K. Allocation still matters: \vacf\ with optimised positions stays $2.43$ $[2.18,2.67]$\,dB above uniform voxels with optimised positions on vessels at 10K and $0.56$\,dB at 50K. Among anatomy-aware initialisations the differences shrink: \vacf\ and the oracle are within 0.15\,dB of each other, and the anatomy heuristic is 0.19\,dB higher than \vacf\ on vessels but 0.30\,dB lower on thin vessels and 0.27\,dB lower on the foreground at 10K (all $p<0.01$); at 50K none of these differences is significant. \vacf\ with optimised positions exceeds per-case 3DGS at the same budget on vessels ($+1.12$ $[0.97,1.26]$\,dB) and thin vessels ($+0.81$\,dB) at 10K, is level at 50K, and is 0.15--0.22\,dB lower on the foreground. It is also faster: at 25K Gaussians on 10 test patients, \vacf\ with optimised positions reaches 41.9\,dB on vessels in 118\,s, whereas per-case 3DGS reaches 41.3\,dB in 248\,s and attains the former value within its 30K iterations in only one of the ten patients (Fig.~\ref{fig:positions}a).

\begin{table*}[!tp]
\caption{Frozen versus optimised positions on the MSD test set ($n=20$). Each entry is PSNR (dB) with positions frozen / optimised jointly with the other attributes (15K iterations, no densification). Per-case 3DGS grows its Gaussians from scratch by budget-capped densification (30K iterations). Frozen-position values come from a repeat run of the protocol and differ from Table~\ref{tab:main} by at most 0.02\,dB in the mean. Time at 25K Gaussians (10 patients, RTX 4090, path tracing excluded): \vacf\ with optimised positions 118\,s (15K iterations), per-case 3DGS 248\,s (30K iterations).}\label{tab:positions}
\centering\footnotesize\setlength{\tabcolsep}{4pt}
\begin{tabular}{l cc cc cc}
\toprule
& \multicolumn{2}{c}{Vessels} & \multicolumn{2}{c}{Thin vessels} & \multicolumn{2}{c}{Foreground} \\
\cmidrule(lr){2-3}\cmidrule(lr){4-5}\cmidrule(lr){6-7}
Initial allocation & 10K & 50K & 10K & 50K & 10K & 50K \\
\midrule
Uniform voxels & 33.07 / 39.02 & 36.97 / 41.79 & 31.79 / 39.48 & 34.62 / 41.91 & 35.90 / 39.95 & 38.15 / 40.97 \\
Anatomy heuristic & 39.49 / 41.65 & 41.78 / 42.29 & 37.10 / 41.98 & 41.35 / 43.10 & 38.22 / 40.63 & 40.71 / 41.25 \\
\vacf & 39.53 / 41.46 & 41.84 / 42.34 & 39.10 / 42.28 & 42.00 / 43.07 & 39.42 / 40.90 & 40.93 / 41.29 \\
Oracle$^\dagger$ & 40.14 / 41.58 & 41.89 / 42.30 & 39.41 / 42.38 & 42.56 / 43.26 & 40.14 / 41.05 & 41.10 / 41.21 \\
\midrule
Per-case 3DGS$^\dagger$ & 40.33 & 42.22 & 41.47 & 43.21 & 41.05 & 41.51 \\
\bottomrule
\end{tabular}
\end{table*}

\begin{figure*}[!tp]\centering
\includegraphics[width=\textwidth]{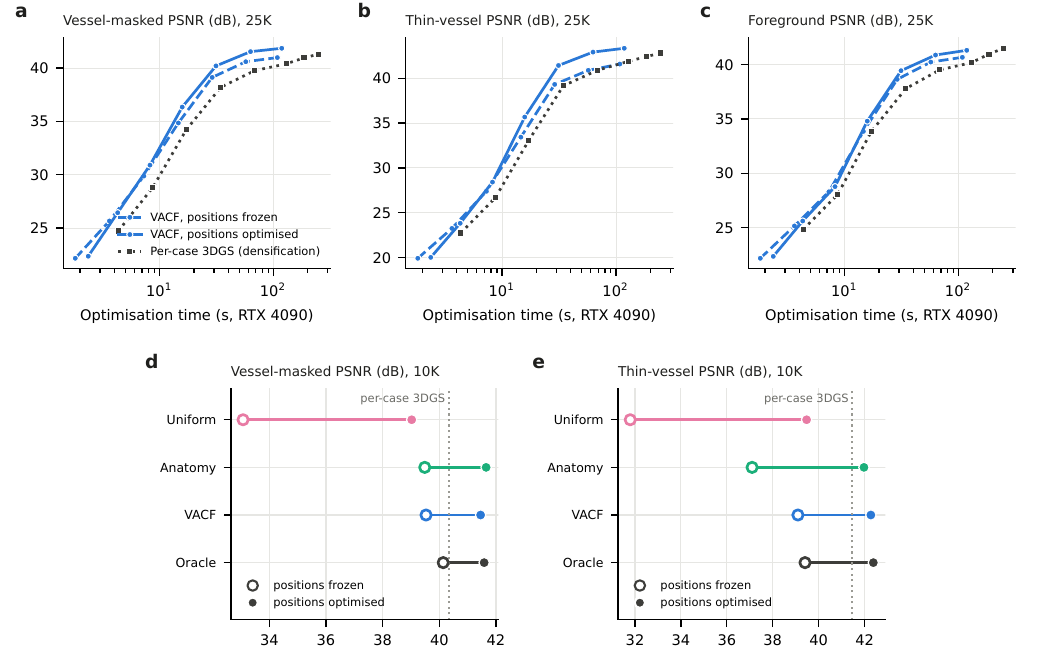}
\caption{Allocation as initialisation for position optimisation. (a--c) Test-view PSNR against optimisation time at 25K Gaussians (means over 10 MSD test patients; evaluation time excluded): \vacf\ allocation with positions frozen (the evaluation protocol) or optimised jointly with the other attributes (15K iterations), and per-case 3DGS grown by budget-capped densification (30K iterations). (d, e) Vessel-masked and thin-vessel PSNR at 10K for four initial allocations with positions frozen (open circles) and optimised (filled circles); dotted line: per-case 3DGS. Position optimisation helps every allocation, but uniform initialisation remains far behind, and \vacf\ with optimised positions exceeds per-case 3DGS on vessels.}\label{fig:positions}
\end{figure*}

\subsection{Rendering presets: strengths and limits of each implementation}\label{sec:res_presets}
Table~\ref{tab:presets} and Fig.~\ref{fig:presets} compare the allocators on the same 10 test patients under the three presets; all allocators were fitted or trained on preset A.

\begin{table}[!tbp]
\caption{Rendering presets, the same 10 MSD test patients. Preset B changes colours, parenchyma density and lighting while keeping a label-aware TF; preset C is a label-free TF that classifies intensities relative to the liver median, so appearance no longer follows the labels. Allocators were fitted or trained on preset A; their inputs are recomputed under each preset, and the oracle uses teachers of that preset. PSNR in dB.}\label{tab:presets}
\centering\footnotesize\setlength{\tabcolsep}{2.4pt}
\begin{tabular}{l cc cc cc}
\toprule
& \multicolumn{2}{c}{Vessels} & \multicolumn{2}{c}{Thin} & \multicolumn{2}{c}{Foreground} \\
\cmidrule(lr){2-3}\cmidrule(lr){4-5}\cmidrule(lr){6-7}
Allocator & 10K & 50K & 10K & 50K & 10K & 50K \\
\midrule
\multicolumn{7}{l}{\emph{Preset A (fitting preset)}} \\
Anatomy heuristic & 39.45 & 41.47 & 37.38 & 41.55 & 38.31 & 40.77 \\
\vacf & 39.31 & 41.47 & 39.09 & 42.12 & 39.42 & 40.97 \\
Neural prior (ours) & \textbf{39.98} & \textbf{41.71} & \textbf{39.54} & \textbf{42.49} & \textbf{40.16} & \textbf{41.10} \\
Neural prior w/o TF inputs & 39.78 & 41.69 & 39.12 & 42.47 & 40.09 & 41.05 \\
Oracle density$^\dagger$ & 39.84 & 41.56 & 39.33 & 42.72 & 40.15 & 41.10 \\
\midrule
\multicolumn{7}{l}{\emph{Preset B: unseen colours and lighting}} \\
Anatomy heuristic & 39.05 & 41.43 & 35.98 & 41.45 & 37.52 & 39.98 \\
\vacf & 40.04 & 41.67 & \textbf{39.50} & 42.88 & 38.91 & 40.30 \\
Neural prior (ours) & 39.15 & 41.62 & 36.75 & 42.82 & 39.23 & 40.43 \\
Neural prior w/o TF inputs & \textbf{40.24} & \textbf{41.70} & 39.38 & \textbf{43.53} & \textbf{39.60} & \textbf{40.54} \\
Oracle density$^\dagger$ & 40.09 & 41.67 & 39.61 & 43.76 & 39.62 & 40.57 \\
\midrule
\multicolumn{7}{l}{\emph{Preset C: unseen label-free TF}} \\
Anatomy heuristic & 37.46 & 40.32 & 37.13 & 40.49 & 34.73 & 39.36 \\
\vacf & 37.61 & 40.44 & 37.43 & 41.11 & 35.22 & 39.52 \\
Neural prior (ours) & \textbf{38.61} & \textbf{41.04} & \textbf{39.85} & \textbf{43.07} & \textbf{38.66} & \textbf{41.10} \\
Neural prior w/o TF inputs & 38.51 & 40.92 & 39.26 & 42.39 & 37.00 & 40.20 \\
Oracle density$^\dagger$ & 38.53 & 41.20 & 39.70 & 42.99 & 39.06 & 41.30 \\
\bottomrule
\end{tabular}
\end{table}

\begin{figure*}[!tp]\centering
\includegraphics[width=\textwidth]{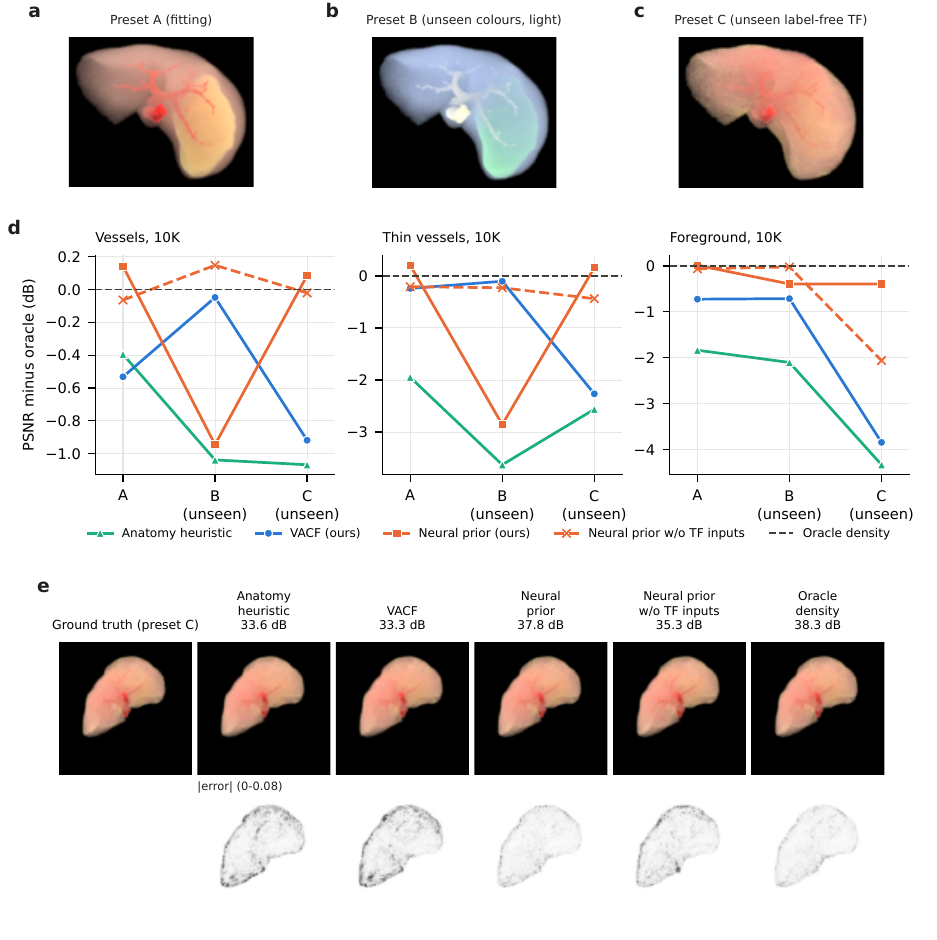}
\caption{Rendering presets (10 MSD test patients; all allocators fitted or trained on preset A). (a--c) One test patient under preset A, preset B (unseen colours, extinctions and lighting) and preset C (unseen label-free TF). (d) PSNR relative to the oracle of each preset at 10K Gaussians on vessels, thin vessels and the foreground. \vacf\ stays at the oracle level under preset B but not under preset C, where appearance no longer follows the labels; the neural prior with TF inputs shows the opposite pattern; without TF inputs it behaves like \vacf. (e) Preset-C renderings of one view at 10K Gaussians (foreground PSNR) and absolute errors: label-driven allocation (anatomy heuristic, \vacf) leaves large errors along the liver surface, where the preset-C oracle spends almost a third of its budget.}\label{fig:presets}
\end{figure*}

\paragraph{Unseen colours, extinctions and lighting (preset B).} \vacf\ remains at the oracle level on vessels and thin vessels ($-0.05$ $[-0.33,0.23]$ and $-0.10$ $[-0.53,0.32]$\,dB at 10K; n.s.\ at all budgets) and is 0.2--0.7\,dB below it on the foreground, and it beats the anatomy heuristic by $+0.99$, $+0.52$ and $+0.24$\,dB on vessels and by $+3.52$, $+2.67$ and $+1.43$\,dB on thin vessels. The neural prior does not transfer: it is $2.85$ $[2.22,3.53]$\,dB below the oracle on thin vessels at 10K, $1.56$ at 25K and $0.94$ at 50K (all Holm-adjusted $p=0.02$). Retrained without the TF inputs, it again does not differ significantly from the oracle on any of the three metrics (thin vessels $-0.23$ $[-0.75,0.28]$\,dB at 10K).

\paragraph{Unseen label-free TF (preset C).} Here the pattern reverses. \vacf\ and the anatomy heuristic fall well below the oracle, by $3.84$ $[3.47,4.21]$\,dB on the foreground and $2.26$ $[1.39,3.20]$\,dB on thin vessels at 10K for \vacf\ (significant at every budget), and \vacf\ is no better than the heuristic. A label-free variant of \vacf\ in which the tissue fractions are replaced by the TF extinction and gradient is better on the foreground ($+0.73$\,dB at 10K) but worse on vessels ($-1.06$\,dB), and remains far from the oracle. The neural prior with TF inputs, in contrast, is at the oracle level on vessels and thin vessels ($+0.09$ and $+0.16$\,dB at 10K; n.s.) and 0.20--0.40\,dB below it on the foreground, whereas without its TF inputs it loses 1.1--2.1\,dB on the foreground.

\paragraph{Why.} Under preset C, appearance no longer follows the labels: on these 10 patients, 34\% of the voxels rendered as dense vessel (extinction $>0.2\,{\rm mm^{-1}}$) carry no vessel label, 25\% of labelled vessel voxels are not rendered as dense, the share of the total extinction that lies in labelled vessels drops from 32\% under preset A to 18\%, and tumours, which are hypodense, become faint instead of dense. The oracle follows the rendering: under preset C it spends 31\% of its budget in the 2\,mm shell below the liver surface (18\% under preset A), 5\% on labelled vessels (13\%) and 4\% on tumour (12\%). \vacf, whose weights encode the relation between labels and appearance of preset A, still spends 22\% on tumour and 13\% on labelled vessels, and only 9\% on the surface shell; its foreground errors concentrate along the liver surface (Fig.~\ref{fig:presets}e). The neural prior can follow the extinction instead, but it relies on the extinction contrast it saw in training: preset C keeps a vessel-to-parenchyma extinction contrast similar to preset A, whereas preset B lowers it from about 37:1 to about 14:1, and this is where the prior with TF inputs under-serves thin vessels.

\subsection{What drives allocation}\label{sec:res_ablation}
\paragraph{What the oracle knows.} On the test set the oracle density beats uniform voxels by 7.1\,dB of vessel-masked PSNR at 10K, and the anatomy heuristic by 0.66\,dB on vessels, 2.3\,dB on thin vessels and 1.9\,dB on the foreground. To understand this gap we compared the oracle with the heuristic region by region. On training patients, the oracle spends 2.3 times more of its budget on a 2\,mm shell below the liver surface, 3.3 times more on thin vessels and 1.5 times more on tumour, and half as much on the parenchyma surrounding vessels; the same pattern holds on the test patients (Fig.~\ref{fig:oracle}a). A gradient-boosted regression of the cell-wise log-ratio between oracle and heuristic densities on interpretable features explains 42\% of its variance on held-out patients, with depth below the surface as the dominant variable (permutation importance 0.69, next 0.09): the oracle favours cells near the surface and starves deep tissue (Fig.~\ref{fig:oracle}b; densities of one patient in Fig.~\ref{fig:densities}). In a translucent volume, deep cells are seen through attenuating tissue and contribute less to every image, which is the visibility term of \eqref{eq:vis}. The neural prior removes this depth bias and \vacf\ halves it (Fig.~\ref{fig:oracle}b).

\begin{figure}[!tbp]\centering
\includegraphics[width=\columnwidth]{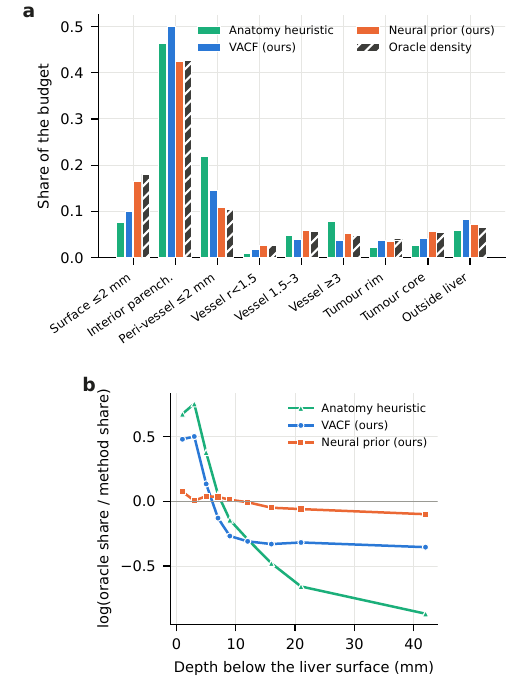}
\caption{How the per-patient oracle spends its budget (20 MSD test patients). (a) Share of the budget per region for the anatomy heuristic, \vacf, the neural prior and the oracle (hatched). Regions are assigned per 2\,mm cell: a 2\,mm shell below the liver surface, interior parenchyma, parenchyma within 2\,mm of a vessel, vessels by radius, tumour rim and core, and cells outside the liver. Relative to the heuristic, the oracle spends more on the surface shell, thin vessels and tumour and much less around vessels. (b) Log-ratio between the oracle's and each method's budget share as a function of depth below the liver surface. The heuristic over-spends deep inside the organ and under-spends near the surface; \vacf\ halves this bias (from about $+0.7$ to $+0.5$ near the surface and from $-0.85$ to $-0.35$ deep inside) and the neural prior removes it.}\label{fig:oracle}
\end{figure}

\begin{figure*}[!tp]\centering
\includegraphics[width=\textwidth]{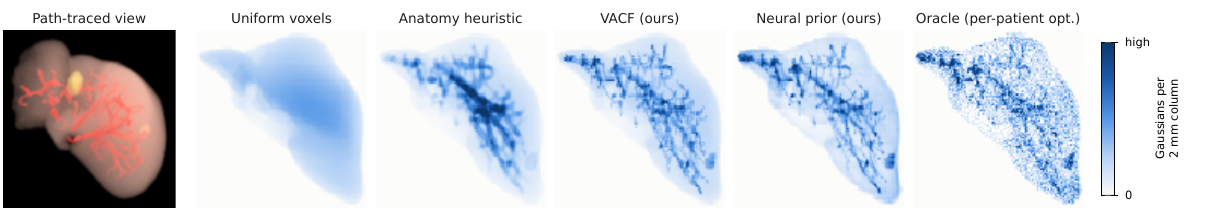}
\caption{Allocation densities of one patient (MSD test patient 372). Left: a path-traced view of the patient (display copy rendered at $1024\times1024$). Right: allocation densities shown as coronal projections of the number of Gaussians per 2\,mm column, on a common colour scale. Uniform voxel allocation spreads the budget over the organ; the anatomy heuristic concentrates it on large vessels and blurs it into their surroundings. The per-patient oracle follows the vessel tree down to small branches while keeping a broad coverage of the parenchyma. \vacf\ and the neural prior reproduce this pattern without per-patient optimisation.}\label{fig:densities}
\end{figure*}

\paragraph{Fitted closed form.} On the 56 training patients the fit gives $\gamma=2.80$, above the range of 1--2 suggested by the squared-error argument alone: optimised scenes favour visible cells even more strongly than attenuation of the error explains, which is why we treat $\gamma$ as an empirical parameter. The simplex weights are $w_{\rm ves}=(0.311,0.321,0.349)$ for the three calibre classes, $w_{\rm tum}=0.019$, $w_{\rm par}=7.7\times10^{-5}$ and $w_g=4.6\times10^{-5}$, with leave-one-patient-out coefficients of variation between 0.1\% and 3.8\%. The three vessel weights are nearly equal; because the vessel terms are volume fractions and the exponent is concave, a cell only partly occupied by a thin vessel still receives a large share relative to its vessel volume. Parenchyma receives Gaussians mainly through visibility and through partial vessel and tumour occupancy of boundary cells.

\paragraph{Terms of the closed form.} Table~\ref{tab:ablation}a refits variants of \vacf\ and evaluates them on the MSD test set. Removing visibility ($\gamma=0$) lowers the foreground PSNR by $0.76$ $[0.64,0.91]$\,dB at 10K and $0.28$\,dB at 50K but barely changes thin vessels ($-0.11$\,dB, n.s.): without visibility the fit moves weight to thin vessels (weight 0.76 of 1), so visibility mainly redistributes the parenchymal budget toward the visible surface. Pooling the three vessel calibres into one fraction makes no difference ($\le0.06$\,dB, n.s.), as the nearly equal weights anticipated. Replacing the model's exponent $1/(1+\beta)=0.51$ by a linear density (refitted $\gamma=1.76$) is slightly better at 10K ($+0.16$\,dB on vessels, $+0.42$\,dB on thin vessels, $p<0.05$) and equivalent at 50K: once the weights are refitted, the exponent is not decisive. Replacing the tissue fractions by the TF extinction loses 1.8\,dB on vessels and 3.8\,dB on thin vessels at 10K under preset A. What matters is that the weights are fitted to optimised allocations, that labels identify the vessels, and that visibility allocates the parenchyma.

\begin{table}[!tbp]
\caption{Ablations on the MSD test set ($n=20$), PSNR in dB. (a) Each \vacf\ variant is refitted on the training patients with the same objective. (b) Allocation cell size; both \vacf\ and the oracle are recomputed at each size. (c) The neural prior retrained without groups of input channels. Bold: best within each block (excluding references). Values come from a repeat run of the protocol and differ from Table~\ref{tab:main} by at most 0.02\,dB in the mean (Section~\ref{sec:impl}).}\label{tab:ablation}
\centering\footnotesize\setlength{\tabcolsep}{2.4pt}
\begin{tabular}{l cc cc cc}
\toprule
& \multicolumn{2}{c}{Vessels} & \multicolumn{2}{c}{Thin} & \multicolumn{2}{c}{Foreground} \\
\cmidrule(lr){2-3}\cmidrule(lr){4-5}\cmidrule(lr){6-7}
Variant & 10K & 50K & 10K & 50K & 10K & 50K \\
\midrule
\multicolumn{7}{l}{\emph{(a) Terms of the closed form}} \\
\vacf\ (full) & 39.53 & 41.84 & 39.10 & 42.00 & 39.42 & \textbf{40.93} \\
\quad w/o visibility ($\gamma=0$) & 39.45 & 41.58 & 38.99 & 41.59 & 38.66 & 40.65 \\
\quad w/o calibre split & 39.53 & 41.82 & 39.16 & 42.00 & 39.40 & 40.89 \\
\quad linear exponent ($1/(1+\beta)\to1$) & \textbf{39.69} & \textbf{41.90} & \textbf{39.52} & \textbf{42.06} & \textbf{39.57} & 40.92 \\
\quad labels $\to$ TF terms (label-free) & 37.69 & 40.77 & 35.33 & 39.41 & 38.21 & 40.21 \\
Anatomy heuristic & 39.49 & 41.78 & 37.10 & 41.35 & 38.22 & 40.71 \\
\midrule
\multicolumn{7}{l}{\emph{(b) Cell size (refitted per size)}} \\
\vacf, 1\,mm cells & \textbf{39.86} & \textbf{41.92} & \textbf{39.98} & \textbf{42.49} & \textbf{39.62} & \textbf{40.97} \\
\vacf, 2\,mm cells & 39.53 & 41.84 & 39.10 & 42.00 & 39.42 & 40.93 \\
\vacf, 4\,mm cells & 38.41 & 41.35 & 36.57 & 40.51 & 38.78 & 40.58 \\
Oracle$^\dagger$, 1\,mm cells & 40.66 & 41.88 & 40.74 & 42.84 & 40.46 & 41.09 \\
Oracle$^\dagger$, 2\,mm cells & 40.14 & 41.89 & 39.41 & 42.56 & 40.14 & 41.10 \\
Oracle$^\dagger$, 4\,mm cells & 38.97 & 41.67 & 36.92 & 41.42 & 39.49 & 40.94 \\
\midrule
\multicolumn{7}{l}{\emph{(c) Inputs of the neural prior}} \\
Neural prior (full) & 40.16 & 42.01 & 39.64 & 42.38 & \textbf{40.13} & 41.06 \\
\quad w/o tissue fractions & \textbf{40.17} & \textbf{42.02} & \textbf{39.65} & 42.44 & 40.09 & 41.03 \\
\quad w/o visibility, depth & 39.94 & 42.00 & 39.46 & \textbf{42.47} & 40.09 & \textbf{41.08} \\
\quad w/o TF inputs & 40.06 & 41.96 & 39.30 & 42.30 & 40.06 & 41.01 \\
\quad CT and depth only & 38.21 & 41.29 & 34.69 & 39.38 & 38.74 & 40.47 \\
\bottomrule
\end{tabular}
\end{table}

\paragraph{Inputs of the neural prior.} Removing the explicit tissue fractions does not change the neural prior ($\le0.06$\,dB, n.s.), because with a label-aware TF the TF channels already encode the labels; removing visibility and depth lowers vessel PSNR by 0.22\,dB at 10K (Holm-adjusted $p=0.06$); removing the TF inputs lowers thin-vessel PSNR by 0.33\,dB at 10K ($p<0.01$) in distribution, the price of the robustness to preset B. A network that sees only CT and depth is much worse ($-4.9$\,dB on thin vessels at 10K): image intensity alone does not tell where the rendering needs Gaussians.

\paragraph{Cell size.} \vacf\ was refitted and the oracle recomputed at 1 and 4\,mm (Table~\ref{tab:ablation}b). Allocation on 4\,mm cells is worse than on 2\,mm cells for both (thin vessels $-2.5$\,dB at 10K), because a 4\,mm cell cannot separate a thin vessel from its surroundings. Refining to 1\,mm helps at small budgets: \vacf\ gains $+0.33$ $[0.21,0.44]$\,dB on vessels, $+0.88$ $[0.76,1.00]$\,dB on thin vessels and $+0.20$\,dB on the foreground at 10K (all $p<0.001$), and at 50K only on thin vessels ($+0.50$\,dB); the oracle gains slightly more ($+0.53$, $+1.33$ and $+0.32$\,dB at 10K), so that the gap between them is 0.8\,dB at 10K and vanishes at 50K. The refitted parameters are stable across cell sizes ($\gamma=2.87$ at 1\,mm against 2.80 at 2\,mm; vessel weights $0.43/0.30/0.25$, tumour $0.012$, parenchyma and gradient $<10^{-4}$), so the closed form does not depend on the discretisation, and the 2\,mm results are, if anything, conservative for small budgets.

\subsection{Supporting analysis: one density across budgets}\label{sec:res_theory}
The front end draws positions for every budget from one density per patient. This is supported if optimised allocations are approximately budget-invariant, as Proposition~\ref{prop:alloc} predicts.

\paragraph{Budget invariance.} We compared the normalised allocation densities of budget-capped teachers at 25K, 50K, 100K and 200K on training patients with the variability between two optimisation seeds at the same budget (Fig.~\ref{fig:theory}b). At 8\,mm resolution the Jensen--Shannon divergence between budgets is 0.069 (25K vs 50K), 0.050 (50K vs 100K) and 0.038 (100K vs 200K), against seed floors of 0.078 at 25K and 0.033 at 200K; only the widest pair (25K vs 200K, 0.087) slightly exceeds the floor of the smaller budget. At 4\,mm the same ordering holds (25K vs 200K 0.31 against a floor of 0.35). Over 25K--200K and at 4--8\,mm resolution, optimised allocations are therefore approximately budget-invariant. We did not test invariance at 2\,mm or below 25K; there the support is empirical: a single density per patient, with no budget-specific fitting, gives the results of Sections~\ref{sec:res_main}--\ref{sec:res_positions} at 10K--100K.

\begin{figure}[!tbp]\centering
\includegraphics[width=\columnwidth]{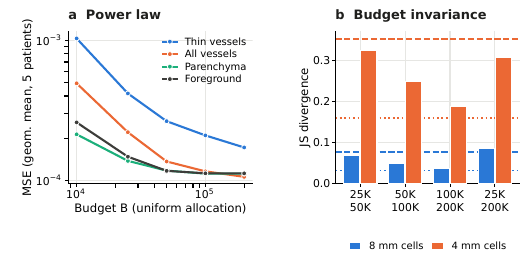}
\caption{Supporting analysis (training patients). (a) Under uniform allocation the MSE falls as a power of the budget, reaching a floor above about 50K Gaussians; fitting $D=AB^{-\beta}+C$ gives $\beta=0.96$ on the foreground, with steeper slopes on vessels and thin vessels than on parenchyma. Geometric means over five patients. (b) Budget invariance: Jensen--Shannon divergence between the normalised allocation densities of budget-capped per-patient models at different budgets (bars), compared with the divergence between two optimisation seeds at the same budget (dashed: 25K; dotted: 200K), at 8\,mm and 4\,mm resolution. Densities at different budgets differ about as much as two seeds do.}\label{fig:theory}
\end{figure}

\paragraph{Power law.} Under uniform allocation, \eqref{eq:distortion} predicts $D\propto B^{-\beta}$. Across five training patients and budgets of 10K--200K, the MSE follows a power law with a floor (Fig.~\ref{fig:theory}a); fitting $D=AB^{-\beta}+C$ gives $\beta=0.96\pm0.05$ on the foreground, which sets the exponent $1/(1+\beta)=0.51$ of \vacf. Slopes differ by region (vessels $1.33\pm0.27$, thin vessels $1.41\pm0.25$, parenchyma $0.90\pm0.08$, tumour $0.64\pm0.39$), so the common-exponent assumption holds only approximately; it implies a slow drift of the optimal allocation towards vessels as $B$ grows, which is consistent with the slight excess of the widest budget pair above the seed floor. The model thus motivates one density for a range of budgets, but it does not guarantee invariance outside the tested range and, as the exponent ablation shows, it does not prescribe the best exponent.

\section{Discussion}\label{sec:discussion}

\paragraph{What the front end solves.} Where the Gaussians of a cinematic CT scene should go under a fixed budget turns out to be predictable from the anatomy and the rendering setup. Learned once from optimised scenes, the front end places a budget of 10K--100K Gaussians as well as per-patient optimisation does (neural prior) or nearly so (\vacf), on held-out and external patients, and raises vessel and thin-vessel quality by several decibels over uniform placement and over pruning a dense voxel-aligned model. It removes the placement search from the per-patient pipeline and gives exact control over the count, which display hardware requires. As an initialisation, the closed-form allocation also makes the remaining per-patient optimisation more effective: starting from the right number of Gaussians in the right places, position optimisation surpasses per-patient densification on vessels at 10K Gaussians and, at 25K, reaches a higher vessel PSNR after 118\,s than densification after 248\,s of optimisation.

\paragraph{Two implementations, different strengths.} The neural prior is the more accurate allocator wherever its inputs look like its training data: it matches the per-patient oracle on the internal and external test sets and, on vessels, under the label-free preset, whose extinction contrast resembles preset A. It learns, however, whatever its TF inputs encode about the training preset, and it under-serves thin vessels when a new TF changes the extinction contrast (preset B). \vacf\ is within a fraction of a decibel of the oracle on vessels with seven readable parameters, which are stable across cell sizes. It transfers to new colours, extinctions and lighting because its tissue terms do not depend on them, and for the same reason it fails when appearance stops following the labels (preset C). Without its TF inputs, the neural prior behaves like \vacf. Each implementation is therefore only as good as the agreement between its inputs and what is rendered: for a viewer with a label-aware TF in which users change colours, \vacf\ or the neural prior without TF inputs is the safe choice; for an intensity-based TF, an appearance-aware allocator is required. That the closed form's failure can be stated in one sentence in advance is an argument for interpretable allocators in clinical software.

\paragraph{What remains unsolved.} The front end decides positions only. Complete scenes still require path-traced training views of each patient and per-patient fitting of colour, opacity and shape (about 2\,min per patient here, in addition to path tracing), so this work does not yet deliver scenes in one forward pass; a feed-forward attribute decoder in the spirit of Render-FM, fed with our positions, is the natural complement. Neither implementation is robust to every TF: training an appearance-aware prior across a distribution of TFs, rather than on one preset, is the obvious next step. Combining the front end with label-aware training as in LAGS, which we did not evaluate, may add further gains on structures. Finally, the study covers one organ and one task, and we measured image quality rather than clinical utility or on-device frame rates.

\paragraph{Limitations.} The oracle is defined by one 3DGS optimiser with generic densification; ground truth comes from one renderer and one camera distribution, and test views share the training views' sphere; the presets are three points in a large space of TFs and lights, and presets B and C were evaluated on 10 patients. Position optimisation was evaluated for \vacf\ and baselines but not for the neural prior, and the neural prior was trained on 2\,mm cells only, although 1\,mm cells improve both \vacf\ and the oracle at small budgets. The MSD test set was evaluated in two rounds and extended afterwards (Section~\ref{sec:selection}); the external patients and the unseen presets, which were never used for any decision, corroborate the conclusions. The budgets studied (10K--100K) overlap those considered for a standalone headset by \citet{lags} (50K--200K), but the relation between budget and frame rate depends on the device and renderer, and we did not measure it.

\section{Conclusion}
We separated the question of where a fixed budget of Gaussians should go from the per-patient optimisation that usually answers it, and showed that a label-aware allocation front end, learned once from optimised scenes, answers it for new patients without per-patient optimisation. A population-trained neural prior matches the allocations of per-patient optimisation on internal and external patients; an interpretable visibility-aware closed form comes within a fraction of a decibel on vessels in under a second; both raise anatomical rendering quality far above uniform placement and pruning at equal budgets, and, as an initialisation, the closed form gives position optimisation a better start than densification. Rendering presets delimit when each implementation can be trusted. Attribute fitting remains per patient; deciding where the Gaussians go no longer has to.

\section*{CRediT authorship contribution statement}
\textbf{Bowen Xiang:} Conceptualization, Methodology, Software, Validation, Formal analysis, Investigation, Data curation, Writing -- original draft, Writing -- review \& editing, Visualization, Project administration. \textbf{Dingjie Su:} Writing -- review \& editing, Visualization. \textbf{Jon S. Heiselman:} Writing -- review \& editing, Visualization. \textbf{Michael I. Miga:} Writing -- review \& editing, Visualization.

\section*{Declaration of competing interest}
The authors declare that they have no known competing financial interests or personal relationships that could have appeared to influence the work reported in this paper.

\section*{Funding}
This research did not receive any specific grant from funding agencies in the public, commercial, or not-for-profit sectors.

\section*{Declaration of generative AI and AI-assisted technologies in the manuscript preparation process}
During the preparation of this work the authors used Claude (Anthropic) to assist with analysis and figure code, drafting and editing of the manuscript text, and checking of references. After using this tool, the authors reviewed and edited the content as needed and take full responsibility for the content of the published article.

\section*{Data and code availability}
MSD Task08 is publicly available \citep{msd}. 3D-IRCADb-01 is available from IRCAD under a CC BY-NC-ND 4.0 licence \citep{ircadb}; in accordance with that licence we do not redistribute renderings or other data derived from it. Code, fitted parameters and the trained networks will be made publicly available upon publication.

\appendix
\section{Development notes}\label{app:notes}
Four pitfalls shaped the protocol and are reported for reproducibility. (i) An early go/stop experiment, run with 64 training views at 256\,spp and 7K-iteration fits, found the oracle within 0.1\,dB of the anatomy heuristic; noise-limited ground truth (about 39\,dB on the liver at 256\,spp) and view-limited overfitting (training-view PSNR rising while test PSNR fell as models grew) had compressed all differences. With 256 training views at 1024\,spp, a real oracle gap appeared. (ii) Per-patient teachers barely densify on soft translucent images with default thresholds, so their densities mostly reflected their initialisation until absolute-gradient densification at a lower threshold was used. (iii) Smoothing the teacher count maps to suppress sampling noise moved thin-vessel mass into neighbouring cells (thin-vessel share of the budget 2.4\%$\to$0.5\%), so allocators fitted to smoothed targets, including the first-round closed form evaluated on the test set, under-served thin vessels; the raw-count likelihood is unbiased and the model itself provides the smoothing. (iv) A mismatch between the cells over which the neural prior was trained and those over which it was normalised at inference leaked a fifth of its budget outside the organ (Section~\ref{sec:selection}); allocation densities should be audited by region before they are evaluated by image quality.

\bibliographystyle{elsarticle-harv}
\bibliography{refs}
\end{document}